\documentclass{article} 
\usepackage{nips15submit_e,times}
\nipsfinalcopy
\usepackage{hyperref}
\usepackage{url}
\usepackage{color,soul}
\usepackage{times}
\usepackage{graphicx} 
\usepackage{subfigure} 

\usepackage{algorithm}
\usepackage{algorithmic}

\usepackage{hyperref}

\usepackage{helvet}
\usepackage{courier}

\usepackage{amssymb}
\usepackage{amsthm}
\usepackage{amsmath}
\usepackage{natbib}

\usepackage{caption}

\usepackage{fullpage}

\title{Bootstrapping Skills}

\author{
Daniel J. Mankowitz \\
Technion - Israel Institute of Technology\\
Electrical Engineering Department, The Technion - Israel Institute of Technology, Haifa 32000, Israel \\
\texttt{danielm@tx.technion.ac.il} \\
\And
Timothy A. Mann \\
Google Deepmind \\
London, UK \\
\texttt{timothymann@google.com} \\
\AND
Shie Mannor \\
Technion - Israel Institute of Technology\\
Electrical Engineering Department, The Technion - Israel Institute of Technology, Haifa 32000, Israel \\
\texttt{shie@ee.technion.ac.il} \\
}

\newtheorem{define}{Definition}
\newtheorem{lemma}{Lemma}
\newtheorem{theorem}{Theorem}

\newcommand{\Algorithm}{Learning Skills via Bootstrapping}
\newcommand{\Alg}{LSB}
\newcommand{\TEA}{TEA}
\newcommand{\TEAs}{\TEA s}

\begin{document}

\maketitle

\begin{abstract}
The monolithic approach to policy representation in Markov Decision Processes (MDPs) looks for a single policy that can be represented as a function from states to actions. For the monolithic approach to succeed (and this is not always possible), a complex feature representation is often necessary since the policy is a complex object that has to prescribe what actions to take all over the state space. This is especially true in large domains with complicated dynamics. It is also computationally inefficient to both learn and plan in MDPs using a complex monolithic approach. We present a different approach where we restrict the policy space to policies that can be represented as combinations of simpler, parameterized skills---a type of temporally extended action, with a simple policy representation. 
We introduce Learning Skills via Bootstrapping (LSB) that can use a broad family of Reinforcement
Learning (RL) algorithms as a ``black box'' to iteratively learn parametrized skills. Initially, the
learned skills are short-sighted but each iteration of the algorithm allows the skills to bootstrap
off one another, improving each skill in the process. We prove that this bootstrapping process returns a near-optimal policy. Furthermore, our experiments demonstrate that LSB can solve MDPs that, given the same representational power, could not be solved by a monolithic approach. Thus, planning with learned skills results in better policies without requiring complex policy representations.
\end{abstract}

%
%


\section{Introduction}

State-of-the-art Reinforcement Learning (RL) algorithms need to produce compact solutions to large or continuous state Markov Decision Processes (MDPs), where a solution, called a policy, generates an action when presented with the current state. One such approach to producing compact solutions is linear function approximation.

MDPs are important for both planning and learning in \textit{Reinforcement Learning (RL)}. The RL planning problem uses an MDP model to derive a policy that maximizes the sum of rewards received, while the RL learning problem learns an MDP model from experience (because the MDP model is unknown in advance). In this paper, we focus on RL planning, and use insights from RL that could be used to scale up to problems that are unsolvable with traditional planning approaches (such as Value Iteration and Policy Iteration (c.f., \cite{Puterman1994}). A general result from machine learning is that the sample complexity of learning increases with the complexity of the representation \cite{Vapnik1998}. In a planning scenario, increased sample complexity directly translates to an increase in computational complexity. Thus monolithic approaches, which learn a single parametric policy that solves the entire MDP,  scale poorly. This is because they often require highly complex feature representations, especially in high-dimensional domains with complicated dynamics, to support near-optimal policies. Instead, we investigate learning a collection of policies over a much simpler feature representation (compact policies) and combine those policies hierarchically.

{\em Generalization}: the ability of a system to perform accurately on unseen data, is important for machine learning in general, and can be achieved in this context by restricting the policy space, resulting in compact policies \cite{Bertsekas1995,Sutton1996}.
Policy Search (PS) algorithms, a form of generalization, learn and maintain a compact policy representation so that the policy generates similar actions in nearby states \cite{Peters2008,Bhatnagar2009}.  

{\em Temporally Extended Actions} \cite[\TEAs,][]{Sutton1999}: Compact policies can be represented and combined hierarchically as \TEAs. \TEAs\ are control structures that execute for multiple timesteps. They have been extensively studied under different names, including skills \cite{Konidaris2009}, macro-actions \cite{Hauskrecht1998,He2011}, and options \cite{Sutton1999}. \TEAs\ are known to speed up the convergence rate of MDP planning algorithms \cite{Sutton1999,Mann2014a}. However, the effectiveness of planning with \TEAs\ depends critically on the given actions. For example, Figure \ref{fig:skill_sets}$a$ depicts an episodic MDP with a single goal region and skills $\{ \sigma_1, \sigma_2, \dots , \sigma_5\}$. In this 2D setting, each skill represents a simple movement in a single, linear direction.  Most of the \TEAs\ move towards the goal region, but $\sigma_5$ moves in the opposite direction of the goal making it impossible to reach. With these \TEAs, we cannot hope to derive a satisfactory solution. On the other hand, if one of the \TEAs\ takes the agent directly to the goal (Figure \ref{fig:skill_sets}$b$, the monolithic approach), then planning becomes trivial. Notice, however, that this \TEA\ may be quite complex, and therefore difficult to learn since, in this 2D setting, it represents non-linear movements in multiple directions.

\noindent\begin{minipage}{.5\textwidth}
\centering
\includegraphics[width=0.4\textwidth]{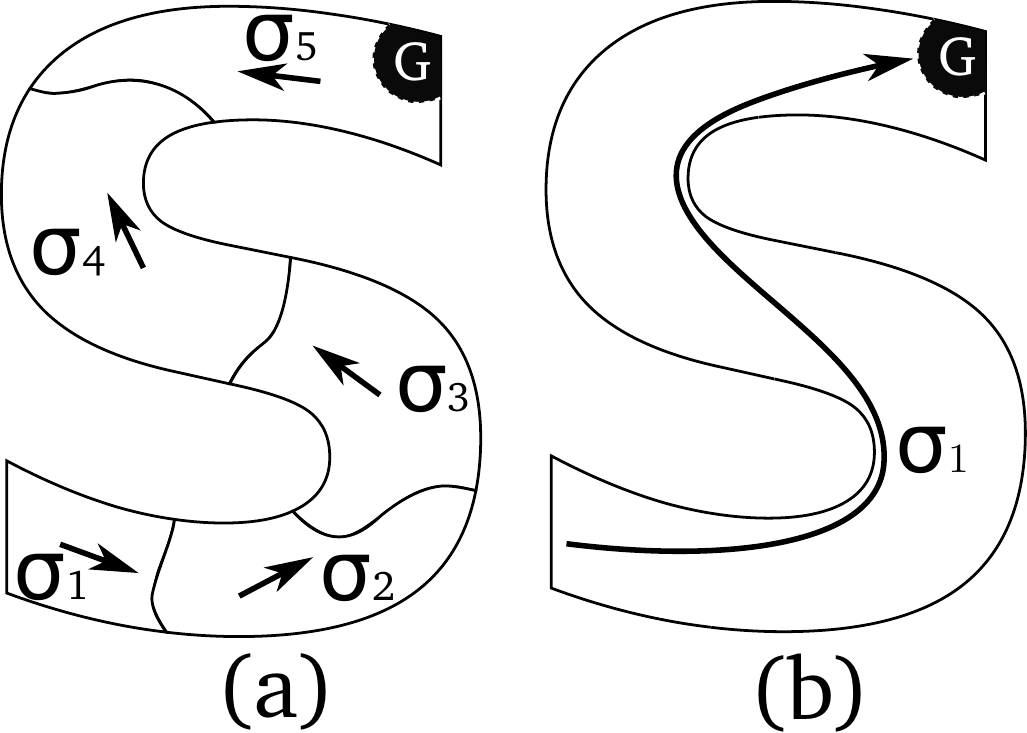} 
\captionof{figure}{\TEAs\ in an episodic MDP with S-shaped state-space and goal region $G$. ($a$) Although most actions move toward the goal, $\sigma_5$ moves away from the goal making it impossible to complete the task. ($b$) Planning becomes trivial when a single \textit{TEA} takes the agent directly to $G$.}
\label{fig:skill_sets}
\end{minipage}%
\hspace{0.2cm}
\begin{minipage}[t]{.5\textwidth}
\centering
\includegraphics[width=0.8\textwidth]{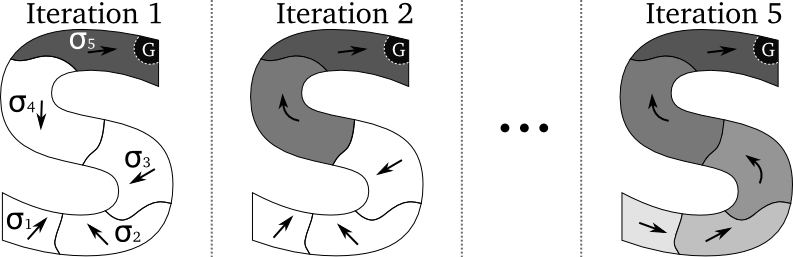}
\captionof{figure}{A partitioning of a target MDP in the pinball domain. Each sub-partition (partition class) $i$ represents the skill MDP $M'_i$.}
\label{fig:bs}
\end{minipage}


Learning a useful set of \TEAs\ has been a topic of intense research \cite{McGovern2001,Moerman2009,Konidaris2009,Brunskill2014, Hauskrecht1998}. However, prior work suffers from the following drawbacks: (1) lack of theoretical analysis guaranteeing that the derived policy will be near-optimal in continuous state MDPs, (2) the process of learning \TEAs\ is so expensive that it needs to be ammortized over a sequence of MDPs, (3) the approach is not applicable to MDPs with large or continuous state-spaces, or (4) the learned \TEAs\ do not generalize over the state-space. We provide the first theoretical guarantees for iteratively learning a set of simple, generalizable parametric \TEAs\ (skills) in a continuous state MDP. The learned \TEAs\ solve the given tasks in a near-optimal manner.



{\bf Skills: Generalization \& Temporal Abstraction:} Skills are \TEAs\ defined over a parametrized policy. Thus, they incorporate both temporal abstraction and generalization. As \TEAs, skills are closely related to options \cite{Sutton1999} developed in the RL literature. In fact, skills, as defined here, are a special case of options. Therefore, skills inherit many of the useful theoretical properties of options (e.g., \cite{Precup1998}). The main difference between skills and more general options is that skills are based on parametric policies that can be initialized and reused in any region of the state space.

%
%

We introduce a novel meta-algorithm, \Algorithm\ (\Alg), that uses an RL algorithm as a ``black box'' to iteratively learn parametrized skills. The learning algorithm is given a partition of the state-space, and one skill is created for each class in the partition. This is a very weak requirement since any partition could be used, such as a grid. During an iteration, an RL algorithm is used to update each skill. The skills may be initialized arbitrarily, but after the first iteration skills with access to a goal region or non-zero rewards will learn how to exploit those rewards (e.g., Figure \ref{fig:bs}, Iteration 1). On further iterations, the newly acquired skills propagate reward back to other regions of the state-space. Thus, skills that previously had no reward signal bootstrap off of the rewards of other skills (e.g., Figure \ref{fig:bs}, Iterations 2 and 5). Although each skill is only learned over a single partition class, it can be initialized in any state.

It is important to note that this paper deals primarily with learning {\em \TEAs} or \textit{Skills} that aid in both speeding up the convergence rate of RL planning algorithms \cite{Sutton1999,Mann2014a}, as well as enabling larger problems to be solved using skills with simple policy representations. Utilizing simple policy representations is advantageous since this results in better generalization and better sample efficiency. These \textit{skills} represent a misspecified model of the problem since they are not known in advance. By learning skills, we are also therefore inherently tackling the learning problem as we are iteratively correcting a misspecified model.

{\bf Contributions:} Our main contributions are \textbf{(1)} The introduction of \Algorithm\ (\Alg), which requires no additional prior knowledge apart from a partition over the state-space. \textbf{(2)} LSB is the first algorithm for learning skills in continuous state-spaces with theoretical convergence guarantees. \textbf{(3)} Theorem \ref{thm:lsb}, which relates the quality of the policy returned by \Alg\ to the quality of the skills learned by the ``black box'' RL algorithm. \textbf{(4)} Experiments demonstrating that \Alg\ can solve MDPs that, given the same representational power, can not be solved by a policy derived from a monolithic approach. Thus, planning with learned skills allows us to work with simpler representations \cite{Barto2013}, which ultimately allows us to solve larger MDPs.


\section{Background}
Let $M = \langle S, A, P, R, \gamma \rangle$ be an MDP, where $S$ is a (possibly infinite) set of states, $A$ is a finite set of actions, $P$ is a mapping from state-action pairs to probability distributions over next states, $R$ maps each state-action pair to a reward in $[0, 1]$, and $\gamma \in [0, 1)$ is the discount factor. While assuming the rewards are in $[0, 1]$ may seem restrictive, any bounded space can be rescaled so that this assumption holds. A policy $\pi(a|s)$ gives the probability of executing action $a \in A$ from state $s \in S$. 

Let $M$ be an MDP. The value function of a policy $\pi$ with respect to a state $s \in S$ is
$
V^{\pi}_{M}(s) = \mathbb{E} \left[ \sum_{t=1}^{\infty} \gamma^{t-1} R(s_t,a_t) | s_0 = s \right]
$ where the expectation is taken with respect to the trajectory produced by following policy $\pi$. The value function of a policy $\pi$ can also be written recursively as

\begin{equation} \label{eqn:value}
V^{\pi}_{M}(s) = \mathbb{E}_{a \sim \pi(\cdot|s)} \left[ R(s,a) \right] + \gamma \mathbb{E}_{s' \sim P(\cdot|s,\pi)} \left[ V^{\pi}(s') \right] \enspace ,
\end{equation}
which is known as the Bellman equation. The optimal Bellman equation can be written as
$
V^{*}_{M}(s) = \max_a \mathbb{E} \left[ R(s,a) \right] + \gamma \mathbb{E}_{s' \sim P(\cdot|s,\pi)} \left[ V^{*}(s') \right] \enspace .
$
Let $\varepsilon > 0$. We say that a policy $\pi$ is $\varepsilon$-optimal if $V^{\pi}_M(s) \geq V^{*}_M(s) - \varepsilon$ for all $s \in S$. The action-value function of a policy $\pi$ can be defined by 
$
Q^{\pi}_{M}(s,a) = \mathbb{E}_{a \sim \pi(\cdot|s)} \left[ R(s,a) \right] + \gamma \mathbb{E}_{s' \sim P(\cdot|s,\pi)} \left[ V^{\pi}(s') \right] \enspace ,
$
for a state $s \in S$ and an action $a \in A$, and the optimal action-value function is denoted by $Q^{*}_{M}(s, a)$. Throughout this paper, we will drop the dependence on $M$ when it is clear from context.

\section{Skills}

One of the key ideas behind skills is that they may be learned locally, but they can be used throughout the entire state-space. We present a new formal definition for skills and a skill policy.

\begin{define} \label{def:skill}
A {\bf skill} $\sigma$ is defined by a pair $\langle \pi_{\theta}, \beta \rangle$, where $\pi_{\theta}$ is a parametric policy with parameter vector $\theta$ and $\beta : S \rightarrow \{ 0, 1\}$ indicates whether the skill has finished (i.e., $\beta(s) = 1$) or not (i.e., $\beta(s) = 0$) given the current state $s \in S$.
\end{define}


\begin{define} \label{def:skill_policy}
Let $\Sigma$ be a set of $m \geq 1$ skills. A {\bf skill policy} $\mu$ is a mapping $\mu : S\rightarrow [m]$ where $S$ is the state-space and $[m]$ is the index set over skills.
\end{define}

A skill policy selects which skill to initialize from the current state by returning the index of one of the skills. By defining skill policies to select an index (rather than the skill itself), we can use the same policy even as the set of skills is adapting. Next we define a \textit{Skill MDP}, which is a sub-partition of a target MDP as shown in Figure \ref{fig:partitioning}.

\begin{define} \label{def:skill_mdp}
Given a target MDP $M = \langle S, A, P, R, \gamma \rangle$ and value function $V_{M}$, a {\bf Skill MDP} for partition $\mathcal{P}_i$ is an MDP defined by $M_i' = \langle S', A, P', R', \gamma \rangle$ where $S' = \mathcal{P}_i \cup \{ s_T \}$ where $s_T$ is a terminal state and $A$ is the action set from $M$. The transition probability function $P'(s'|s,a)$ and reward function $R'(s, a)$ are defined below.
\\$P'(s'|s,a) =$ \hspace{5.0cm} $R'(s, a) = $
$$
\left\{ \begin{array}{cl} P(s'|s,a) & \textrm{if } s \in \mathcal{P}_i \wedge s' \in \mathcal{P}_i \\ \sum_{y \in S \backslash \mathcal{P}_i} P(y|s,a) & \textrm{if } s \in \mathcal{P}_i \wedge s' = s_T \\ 1 & \textrm{if } s = s_T \wedge s' = s_T \\ 0 & \textrm{if } s = s_T \wedge s' \neq s_T  \end{array} \right. ,
\left\{ \begin{array}{cl} 
0 & \textrm{if } s = s_T \\ 

\sum\limits_{s' \in \mathcal{P}_i} \hspace{-0.5em} P(s'|s,a) R(s,a) & \textrm{if } s \neq s_T \wedge s' \neq s_T \\ 

\sum\limits_{y \in S \backslash \mathcal{P}_i} \hspace{-1.2em} \psi(s,a,y)  & \textrm{if } s \neq s_T \wedge s' = s_T \end{array} \right. \enspace ,
$$
where $\psi(s, a, y) = P(y|s,a)\left( R(s,a) + \gamma V_{M}(y) \right)$, and $\gamma$ is the discount factor from $M$.
\end{define}

A Skill MDP $M_i'$ is an episodic MDP that terminates once the agent escapes from $\mathcal{P}_i$ and upon terminating receives a reward equal to the value of the state the agent would have transitioned to in the target MDP. Therefore, we construct a modified MDP called a Skill MDP and apply a planning or RL algorithm to solve it. The resulting solution is a skill. Each Skill MDP $M_i'$ is defined within the partition $\mathcal{P}_i$. 

Given a good set of skills, planning can be significantly faster \cite{Sutton1999,Mann2014a}. However, in many domains we may not be given a good set of skills. Therefore it is necessary to learn this set of skills given the unsatisfactory skill set. In the next section, we introduce an algorithm for dynamically improving skills via bootstrapping.

%
%
\section{\Algorithm\ (\Alg) Algorithm}

\noindent\begin{minipage}{.6\textwidth}
\captionof{algorithm}{\textbf{\Algorithm\ (\Alg)}}\label{alg:slb}
\label{alg:slb}
\begin{algorithmic}[1]
\REQUIRE $M$ \COMMENT{Target MDP}, $\mathcal{P}$ \COMMENT{Partitioning of $S$},\\ $K$ \COMMENT{\# Iterations}
\STATE $m \leftarrow | \mathcal{P} |$ \COMMENT{\# of partitions.} \label{alg:slb:num_skills}
\STATE $\mu(s) = \arg \max_{i \in [m]} \mathbb{I} \{ s \in \mathcal{P}_i \}$ \label{alg:slb:skill_policy}
\STATE Initialize $\Sigma$ with $m$ skills. \COMMENT{1 skill per partition.} \label{alg:slb:init_skills}
\FOR[Do $K$ iterations.]{$k = 1, 2, \dots, K$} \label{alg:slb:iters}
\FOR[One update per skill.]{$i = 1, 2, \dots, m$} \label{alg:slb:updates}
	
	\STATE \textbf{Policy Evaluation:}
	\STATE Evaluate $\mu$ with $\Sigma$ to obtain $V^{\langle \mu, \Sigma \rangle}_{M}$ \label{alg:slb:evaluate}
	\STATE \textbf{Skill Update:}
	\STATE Construct Skill MDP $M_{i}'$ from $M$ \& $V^{\langle \mu, \Sigma \rangle}_{M}$ 
	\STATE Solve $M_{i}'$ obtaining policy $\pi_{\theta}$ 
	\STATE $\sigma_i' \leftarrow \langle \pi_{\theta}, \beta_i \rangle$ \label{alg:slb:update_skill}
	\STATE Replace $\sigma_i$ in $\Sigma$ by $\sigma_i'$ \label{alg:slb:update_sigma}
\ENDFOR \label{alg:slb:updates_done}
\ENDFOR \label{alg:slb:iters_done}
\STATE \textbf{return} $\langle \mu, \Sigma \rangle$
\end{algorithmic}
\end{minipage}%
\hspace{0.2cm}
\begin{minipage}[t]{.4\textwidth}
\centering
\includegraphics[width=0.6\textwidth]{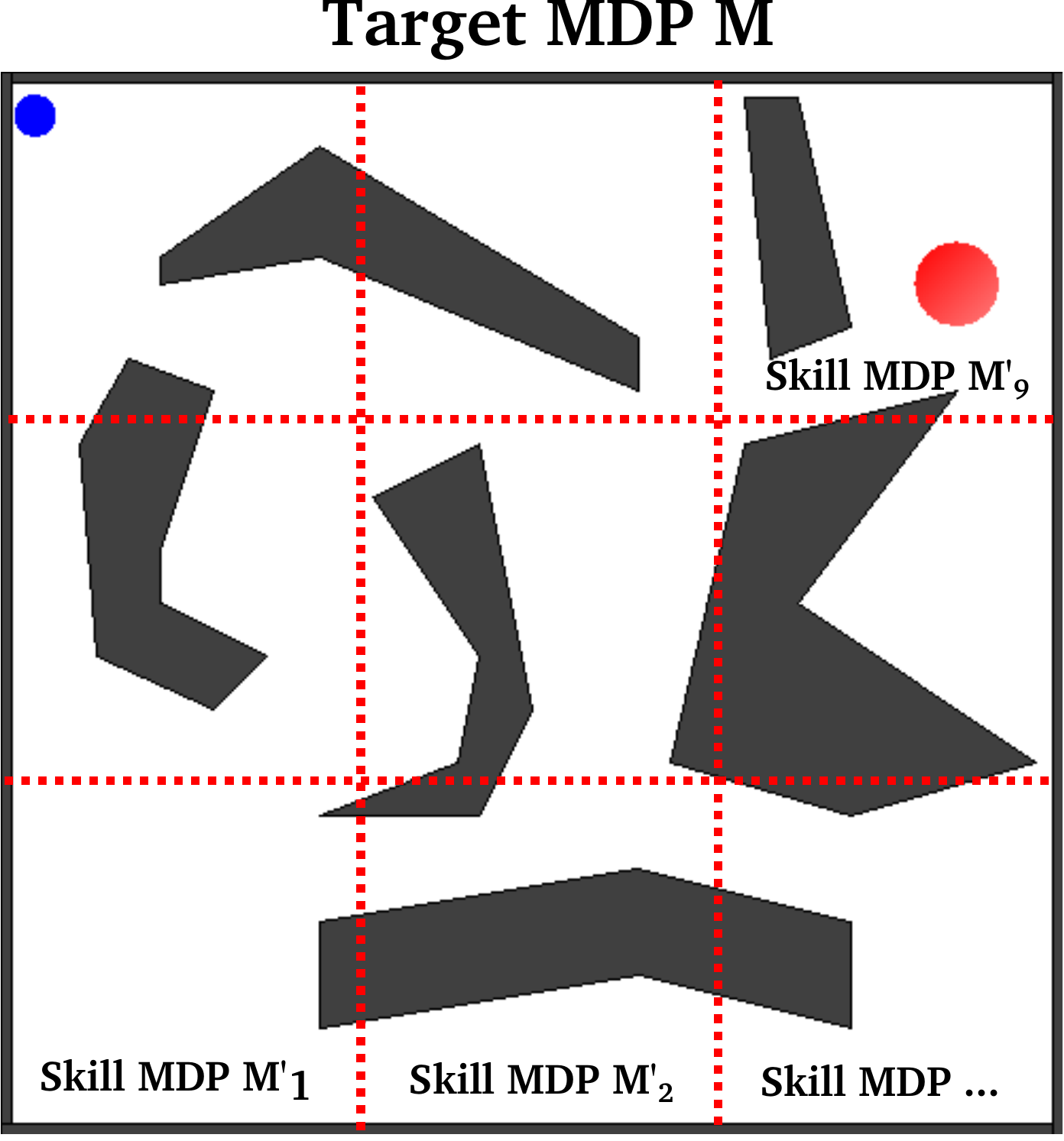}
\captionof{figure}{A partitioning of a target MDP in the pinball domain. Each sub-partition (partition class) $i$ represents the skill MDP $M'_i$. Note that, so long as the classes overlap one another and the goal region is within one of the classes, near-optimal convergence is guaranteed. Therefore, the entire state-space does not have to be partitioned.}
\label{fig:partitioning}
\end{minipage}

\Algorithm\ (\Alg, Algorithm \ref{alg:slb}) takes a target MDP $M$, a partition $\mathcal{P}$ over the state-space and a number of iterations $K \geq 1$ and returns a pair $\langle \mu, \Sigma \rangle$ containing a skill policy $\mu$ and a set of skills $\Sigma$. The number of skills $m = | \mathcal{P} |$ is equal to the number of classes in the partition $\mathcal{P}$ 
(line 
1). The skill policy $\mu$ returned by \Alg\ is defined 
(line 
2)
by
\begin{equation}
\mu(s) = \arg \max_{i \in [m]} \mathbb{I} \left\{ s \in \mathcal{P}_i \right\} \enspace ,
\end{equation}
\vspace{-0.3cm}

where $\mathbb{I} \{ \cdot \}$ is the indicator function returning $1$ if its argument is true and $0$ otherwise and $\mathcal{P}_i$ denotes the $i^{\rm th}$ class in the partition $\mathcal{P}$. Thus $\mu$ simply returns the index of the skill associated with the partition class containing the current state. On line 
3, 
\Alg\ could either initialize $\Sigma$ with skills that we believe might be useful or initialize them arbitrarily, depending on our level of prior knowledge. 

Next 
(lines 
4 -- 14),
\Alg\ performs $K$ iterations. In each iteration, \Alg\ updates the skills in $\Sigma$ 
(lines 
5 -- 13).
Remember that the value of a skill depends on how it is combined with other skills (e.g., Figure \ref{fig:skill_sets}a failed because a single \TEA\ prevented reaching the goal). If we allowed all skills to change simultaneously, the skills could not reliably bootstrap off of each other. Therefore, \Alg\ updates each skill individually. Multiple iterations are needed so that the skill set can converge (Figure \ref{fig:bs}).


The process of updating a skill 
(lines 
6 -- 12)
starts by evaluating $\mu$ with the current skill set $\Sigma$ 
(line 
6).
Any number of policy evaluation algorithms could be used here, such as TD$(\lambda)$ with function approximation \cite{Sutton1998} or LSTD \cite{Boyan2002}, modified to be used with skills. In our experiments, we used a straighforward variant of LSTD \cite{Sorg2010}. Then we use the target MDP $M$ to construct a Skill MDP $M'$ 
(line 
9).
Next, \Alg\ uses a planning or RL algorithm to approximately solve the Skill MDP $M'$ returning a parametrized policy $\pi_{\theta}$ 
(line 
10).
Any planning or RL algorithm for regular MDPs could fill this role provided that it produces a parametrized policy. However, in our experiments, we used a simple actor-critic PG algorithm, unless otherwise stated. Then a new skill $\sigma_i' = \langle \pi_\theta, \beta_i \rangle$ is created 
(line 
11)
where $\pi_\theta$ is the policy derived on line 
10
and $\beta_i(s) = \left\{ \begin{array}{ll} 0 & \textrm{if } s \in \mathcal{P}_i \\ 1 & \textrm{otherwise} \end{array} \right.$. The definition of $\beta_i$ means that the skill will terminate only if it leaves the $i^{\rm th}$ partition. Finally, we update the skill set $\Sigma$ by replacing the $i^{\rm th}$ skill with $\sigma_i'$ 
(line 
12). It is important to note that in \Alg, updating a skill is equivalent to solving a Skill MDP. 

\section{Analysis of \Alg}

We provide the first convergence guarantee for iteratively learning skills in a continuous state MDP using \Alg\ (Lemma 1 and Lemma 2, proven in the supplementary material). We use this guarantee as well as Lemma 2 to prove Theorem \ref{thm:lsb}. This theorem enables us to analyze the quality of the policy returned by \Alg. It turns out that the quality of the policy depends critically on the quality of the skill learning algorithm. An important parameter for determining the quality of a policy returned by \Alg\ is the skill learning error defined below.

\begin{define}
\label{def:local}
Let $\mathcal{P}$ be a partition over the target MDP's state-space. The {\bf skill learning error} is
\begin{equation}
\eta_{\mathcal{P}} = \max_{i \in [m]} \eta_i \enspace ,
\end{equation}
where $\eta_i$ is the smallest $\eta_i \geq 0$, such that\\
$
V^{*}_{M_i'}(s) - V^{\pi_\theta}_{M_i'}(s) \leq \eta_i \enspace ,
$ 
for all $s \in \mathcal{P}_i$ and $\pi_\theta$ is the policy returned by the skill learning algorithm executed on $M_i'$.
\end{define}

The skill learning error quantifies the quality of the Skill MDP solutions returned by our skill learning algorithm. If we used an exact solver to learn skills, then $\eta_{\mathcal{P}} = 0$. However, if we use an approximate solver, then $\eta_{\mathcal{P}}$ will be non-zero and the quality will depend on the partition $\mathcal{P}$. Generally, using finer grain partitions will decrease $\eta_{\mathcal{P}}$. However, Theorem \ref{thm:lsb} reveals that adding too many skills can also negatively impact the returned policy's quality.

\begin{theorem} \label{thm:lsb}
Let $\varepsilon > 0$. If we run \Alg\ with partition $\mathcal{P}$ for $K \geq \log_{\gamma} \left( \varepsilon ( 1-\gamma ) \right)$ iterations, then the algorithm returns policy $\varphi = \langle \mu, \Sigma \rangle$ such that
\begin{equation} \label{eqn:lsb_err}
\Vert V^*_{M} - V^{\varphi}_{M} \Vert_\infty \leq \frac{m\eta_{\mathcal{P}}}{(1-\gamma)^2} + \varepsilon \enspace ,
\end{equation}
where $m$ is the number of classes in $\mathcal{P}$.
\end{theorem}

The proof of Theorem \ref{thm:lsb} is divided into three parts (a complete proof is given in the supplementary material). The main challenge to proving Theorem \ref{thm:lsb} is that updating one skill can have a significant impact on the value of other skills. Our analysis starts by bounding the impact of updating one skill. Note that $\Sigma$ represents a skill set and $\Sigma_i$ represents a skill set where we have updated the $i^{th}$ skill (corresponding to the $i^{th}$ partition class $\mathcal{P}_i$) in the set. (1) First, we show that error between $ V_M^{*}$, the globally optimal value function, and $V_M^{\langle \mu, \Sigma_i \rangle}$, is a contraction when $s \in \mathcal{P}_i$ and is bound by $\Vert V_M^{*} - V_M^{\langle \mu, \Sigma \rangle} \Vert_\infty + \frac{\eta_{\mathcal{P}}}{1-\gamma}$ otherwise (Lemma 1). (2) Next we apply an inductive argument to show that updating all $m$ skills results in a $\gamma$ contraction over the entire state space (Lemma 2). (3) Finally, we apply this contraction recursively, which proves Theorem \ref{thm:lsb}.

This provides the first theoretical guarantees of convergence to a near optimal solution when iteratively learning a set of skills $\Sigma$ in a continuous state space. Theorem \ref{thm:lsb} tells us that when the skill learning error is small, \Alg\ returns a near-optimal policy. The first term on the right hand side of (\ref{eqn:lsb_err}) is the approximation error. This is the loss we pay for the parametrized class of policies that we learn skills over. Since $m$ represents the number of classes defined by the partition, we now have a formal way of analysing the effect of the partitioning structure. In addition, complex skills do not need to be designed by a domain expert; only the partitioning needs to be provided \textit{a-priori}. The second term is the convergence error. It goes to $0$ as the number of iterations $K$ increases.

At first, the guarantee provided by Theorem \ref{thm:lsb} may appear similar to (\cite{Hauskrecht1998}, Theorem 1). However, \cite{Hauskrecht1998}  derive \TEAs\ only at the beginning of the learning process and do not update them. On the other hand, \Alg\ updates its skill set dynamically via bootstrapping. Thus, \Alg\ does not require prior knowledge of the optimal value function.

Theorem \ref{thm:lsb} does not explicitly present the effect of policy evaluation error, which occurs with any approximate policy evaluation technique. However, if the policy evaluation error is bounded by $\nu > 0$, then we can simply replace $\eta_{\mathcal{P}}$ in (\ref{eqn:lsb_err}) with $(\eta_{\mathcal{P}} + \nu)$. Again, smaller policy evaluation error leads to smaller approximation error.

%
%

\section{Experiments and Results}
We performed experiments on three well-known RL benchmarks: Mountain Car (MC), Puddle World (PW) \cite{Sutton1996} and the Pinball domain \cite{Konidaris2009}. The MC domain has similar results to PW and therefore has been moved to the supplementary material. We use two variations for the Pinball domain, namely \textit{maze-world}, which we created, and \textit{pinball-world} which is one of the standard pinball benchmark domains. Our experiments show that, using a simple policy representation, the monolithic approach is unable to adequately solve the tasks in each case as the policy representation is not complex enough. However, \Alg\ can solve these tasks with the same simple policy representation by combining bootstrapped skills.  These domains are simple enough that we can still solve them using richer representations. This allows us to compare \Alg\ to a policy that is very close to optimal. Our experiments demonstrate potential to scale up to higher dimensional domains by combining skills over simple representations.

Recall that \Alg\ is a meta-algorithm. We must provide an algorithm for Policy Evaluation (PE) and skill learning. In our experiments, for the MC and PW domains, we used SMDP-LSTD \cite{Sorg2010} for PE and a modified version of Regular-Gradient Actor-Critic \cite{Bhatnagar2009} for skill learning (see supplementary material for details). In the Pinball domains, we used Nearest-Neighbor Function Approximation (NN-FA) for PE and UCB Random Policy Search (UCB-RPS) for skill learning.

In our experiments, for the MC and PW domains, each skill is simply represented as a probability distribution over actions (independent of the state). We compare their performance to a policy using the same representation that has been derived using the monolithic approach. Each experiment is run for $10$ independent trials. A $2 \times 2$ grid partitioning is used for the skill partition in these domains, unless otherwise stated. Binary-grid features are used to estimate the value function. 
In the pinball domains, each skill is represented by $5$ polynomial features corresponding to each state dimension and a bias term. A $4 \times 1 \times 1 \times 1$ grid-partitioning is used for \textit{maze-world} and a $4 \times 3 \times 1 \times 1$ partitioning is used for \textit{pinball-world}. The value function is represented by a KD-Tree containing $1000$ state-value pairs uniformly sampled in the domain. A value for a particular state is obtained by assigning the value of the nearest neighbor to that state that is contained within the KD-tree. Each experiment in the pinball domain has been run for $5$ independent trials.
These are example representations. In principal, any value function and policy representation that is representative of the domain can be utilized. 



\begin{figure}
\centering
\includegraphics[width=0.7\textwidth]{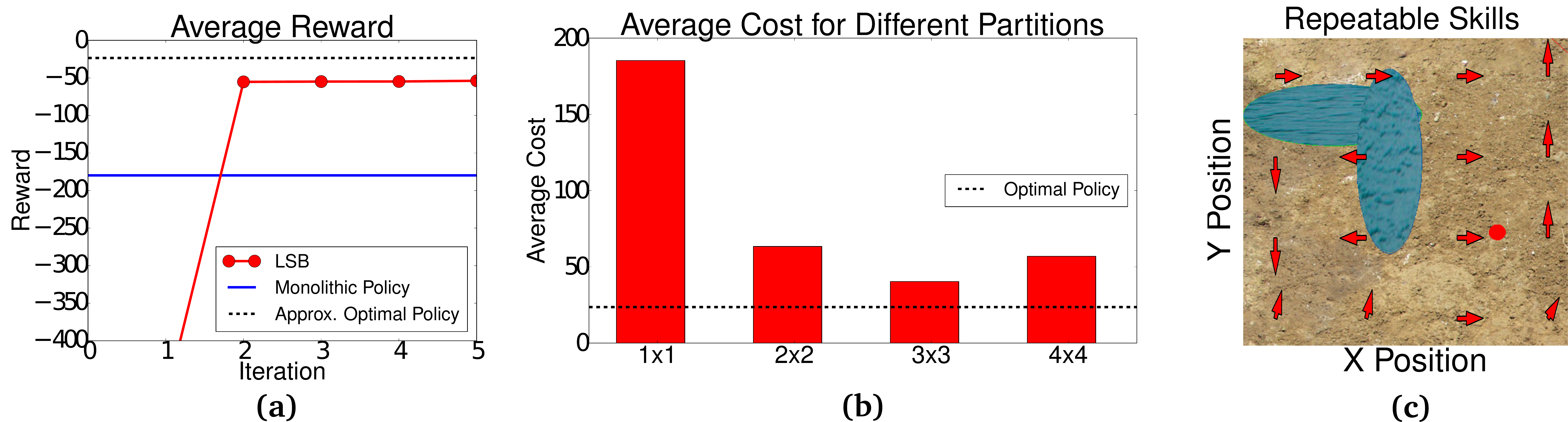} 
\caption{The Puddle World domain: ($a$) The average reward for the \Alg\ algorithm generated by the \Alg\ skill policy. This is compared to the monolithic approach that attempts to solve the global task as well as an approximately optimal policy derived using Q-learning (applied for a huge number of iterations). ($b$) The average cost (negative reward) for each grid partition. ($c$) Repeatable skills plot.}
\label{fig:pw}
\end{figure}

\begin{figure*}
\centering
\includegraphics[width=1.0\textwidth]{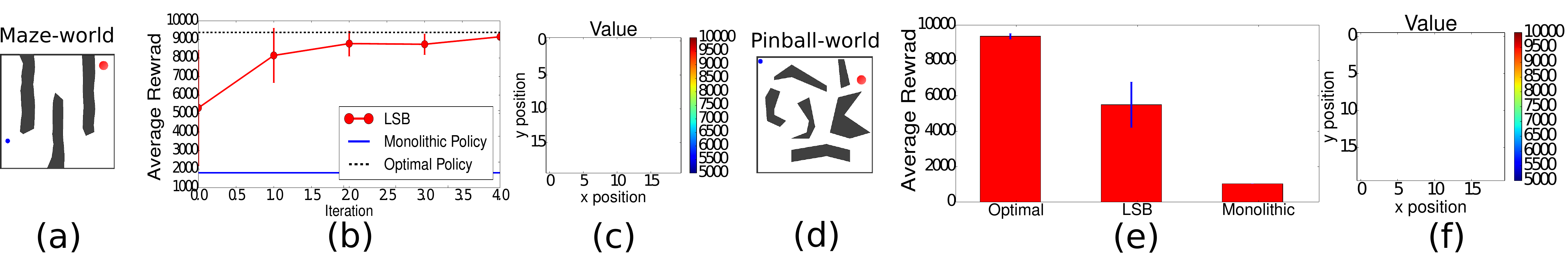}
\caption{The Pinball domains: ($a$) The maze-world domain. ($b$) The average reward for the \Alg\ algorithm generated by the \Alg\ skill policy for \textit{maze-world}. This is compared to the monolithic approach that attempts to solve the global task as well as an approximately optimal policy derived using Approximate Value Iteration (AVI) executed for a huge number of iterations. ($c$) The learned value function for the maze-world domain. ($d$) The pinball-world domain. ($e$) The average reward for the \Alg\ algorithm generated by the \Alg\ skill policy in the \textit{pinball-world} domain. In this domain, \Alg\ converges after a single iteration as we start \Alg\ in the partition containing the goal. ($f$) The learned value function.}
\label{fig:pinball}
\end{figure*}


\subsection{Puddle World}
\label{exp:pw}
Puddle World is a 2-dimensional world containing two puddles. A successful agent should navigate to the goal location, avoiding the puddles. The state space is the $\langle x,y \rangle$ location of the agent. Figure \ref{fig:pw}$a$ compares the monolithic approach with \Alg\ (for a $2 \times 2$ grid partition). The monolithic approach achieves low average reward. However, with the same restricted policy representation, \Alg\ combines a set of skills, resulting in a richer solution space and a higher average reward as seen in Figure \ref{fig:pw}$a$. This is comparable to the approximately optimal average reward attained by executing Approximate Value Iteration (AVI) for a huge number of iterations. In this experiment \Alg\ is not initiated in the partition class containing the goal state but still achieves near-optimal convergence after only $2$ iterations. 


Figure \ref{fig:pw}$b$ compares the performance of different partitions where a $1 \times 1$ grid represents the monolithic approach. The skill learning error $\eta_P$ is significantly smaller for all the partitions greater than $1\times 1$, resulting in lower cost. On the other hand, according to Theorem 1, adding more skills $m$ increases the cost. A tradeoff therefore exists between $\eta_P$ and $m$. In practice, $\eta_P$ tends to dominate $m$. In addition to the tradeoff, the importance of the partition design is evident when analyzing the cost of the $3 \times 3$ and $4 \times 4$ grids. In this scenario, the $3 \times 3$ partition design is better suited to Puddle World than the $4 \times 4$ partition, resulting in lower cost. 






\vspace{-0.4cm}
\subsection{Skill Generalization}
\vspace{-0.2cm}
\label{exp:rep_skills}
In the worst case, the number of skills learned by (\Alg) is based on the partition. However, \Alg\ may learn similar skills in different partition classes, adding redundancy to the skill set. This suggests that skills can be reused in different parts of the state-space, resulting in less skills compared to the number of partition classes.
To validate this intuition, a $4 \times 4$ grid was created for both the Mountain Car and Puddle World domains. We ran \Alg\ using this grid on each domain. Since it is more intuitive to visualize and analyze the reusable skills generated for the 2D Puddle World, we present these skills in a quiver plot superimposed on the Puddle World (Figure \ref{fig:pw}$c$). For each skill, the direction (red arrows in Figure \ref{fig:pw}$c$) is determined by sampling and averaging actions from the skill's probability distribution. As can be seen in Figure \ref{fig:pw}$c$, many of the learned skills have the same direction. These skills can therefore be combined into a single skill and reused throughout the state-space. In this case, the skill-set consisting of $16$ skills can be reduced to a reusable skill-set of $5$ skills (the four cardinal directions, including two skills that are in the approximately north direction). Therefore, skill reuse may further reduce the complexity of a solution.



\vspace{-0.4cm}
\subsection{Pinball}
\vspace{-0.2cm}
\label{sec:pinball}
These experiments have been performed in domains with simple dynamics. We decided to test \Alg\ on a domain with significantly more complicated dynamics, namely Pinball \cite{Konidaris2009}. The goal in Pinball is to direct an agent (the blue ball) to the goal location (the red region). The Pinball domain provides a sterner test for \Alg\ as the velocity at which the agent is travelling needs to be taken into account to circumnavigate obstacles. In addition, collisions with obstacles in the environment are non-linear at obstacle vertices. The state space is the four-tuple $\langle x, y, \dot{x}, \dot{y} \rangle$ where $x, y$ represents the 2D location of the agent, and $\dot{x}, \dot{y}$ represents the velocities in each direction.


Two domains have been utilized, namely \textit{maze-world} and \textit{pinball-world} (Figure \ref{fig:pinball}$a$ and Figure \ref{fig:pinball}$d$ respectively). For \textit{maze-world}, a $4 \times 1 \times 1 \times 1$ grid partitioning has been utilized and therefore $4$ skills need to be learned using \Alg. After running LSB on the maze-world domain, it can be seen in Figure \ref{fig:pinball}$b$ that \Alg\ significantly outperforms the monolithic approach. Note that each skill in \Alg\ has the same parametric representation as the monolithic approach. That is, a five-tuple $\langle 1, x, y, \dot{x}, \dot{y} \rangle$. This simple parametric representation does not have the power to consistently solve maze-world using the monolithic approach. However, using \Alg\, this simple representation is capable of solving the task in a near-optimal fashion as indicated on the average reward graph (Figure \ref{fig:pinball}$b$) and resulting value function (Figure \ref{fig:pinball}$c$).

We also tested \Alg\ on the more challenging pinball-world domain (\ref{fig:pinball}$d$). The same \Alg\ parameters were used as in maze-world, but the provided partitioning was a $4 \times 3 \times 1 \times 1$ grid. Therefore, $12$ skills needed to be learned in this domain. More skills were utilized for this domain since the domain is significantly more complicated than maze-world and a more refined skill-set is required to solve the task. As can be seen in the average reward graph in Figure \ref{fig:pinball}$e$, \Alg\ clearly outperforms the monolithic approach in this domain. It is less than optimal but still manages to sufficiently perform the task (see value function, Figure \ref{fig:pinball}$f$). The drop in performance is due to the complicated obstacle setup, the non-linear dynamics when colliding with obstacle edges and the partition design.


\vspace{-0.4cm}
\section{Discussion}
\vspace{-0.35cm}
In this paper, we introduced an iterative bootstrapping procedure for learning skills. This approach is similar to (and partly inspired by) skill chaining \cite{Konidaris2009}. However, the heuristic approach applied by skill chaining may not produce a near-optimal policy even when the skill learning error is small. We provide theoretical results for \Alg\ that directly relate the quality of the final policy to the skill learning error. \Alg\ is the first algorithm that provides theoretical convergence guarantees whilst iteratively learning a set of skills in a continuous state space. In addition, the theoretical guarantees for \Alg\ enable us to interlace skill learning with Policy Evaluation (PE). We can therefore perform PE whilst learning skills and still converge to a near-optimal solution.

In each of the experiments, \Alg\ converges in very few iterations. This is because we perform policy evaluation in between each skill update, causing the global value function to converge at a fast pace. Initializing \Alg\ in the partition class containing the goal state also results in value being propagated quickly to subsequent partition classes and therefore fast convergence. However, \Alg\ can be initialized from any partition class.


One limitation of \Alg\ is that it learns skills for all partition classes. This is a problem in high-dimensional state-spaces. However, the problem can be overcome, by focusing only on the most important regions of the state-space. One way to identify these regions is by observing an expert's demonstrations \cite{Abbeel2005,Argall2009}. In addition, we could apply self-organizing approaches to facilitate skill reuse \cite{Moerman2009}. Skill reuse can be especially useful for \textit{transfer learning}. Consider a multi-agent environment \cite{Garant2015} where many of the agents may be performing similar tasks which require a similar skill-set. In this environment, skill reuse can facilitate learning complex multi-agent policies (co-learning) with very few samples.

Given a task, \Alg\ can learn and combine skills, based on a set of rules, to solve the task. This structure of learned skills and combination rules forms a \textit{generative action grammar} \cite{Summers2012} which paves the way for building advanced skill structures that are capable of solving complex tasks in different environments and conditions. 

One exciting extension of our work would be to incorporate skill interruption, similar to option interruption. Option interruption involves terminating an option based on an adaptive interruption rule \cite{Sutton1999}. Options are terminated when the value of continuing the current option is lower than the value of switching to a new option. This also implies that partition classes can overlap one another, as the option interruption rule ensures that the option with the best long term value is always being executed. \cite{Mann2014b} interlaced Sutton's interruption rule between iterations of value iteration and proved convergence to a global optimum. In addition, they take advantage of faster convergence rates due to temporal extension by adding a time-based regularization term resulting in a new option interruption rule. However, their results have not yet been extended to use with function approximation. \cite{Comanici2010} have developed a policy gradient technique for learning the termination conditions of options. Their method involves augmentation of the state-space. However, the overall solution converges to a local optimum.


\newpage
\appendix

\section{Appendix}
\subsection{\Alg\ Skill MDP}

The formal definition of a Skill MDP is provided here for completeness.
\begin{define} \label{def:skill_mdp}
Given a target MDP $M = \langle S, A, P, R, \gamma \rangle$ and value function $V_{M}$, a {\bf Skill MDP} for partition $\mathcal{P}_i$ is an MDP defined by $M_i' = \langle S', A, P', R', \gamma \rangle$ where $S' = \mathcal{P}_i \cup \{ s_T \}$ where $s_T$ is a terminal state and $A$ is the action set from $M$. The transition probabilities
$$
P'(s'|s,a) = \left\{ \begin{array}{cl} P(s'|s,a) & \textrm{if } s \in \mathcal{P}_i \wedge s' \in \mathcal{P}_i \\ \sum_{y \in S \backslash \mathcal{P}_i} P(y|s,a) & \textrm{if } s \in \mathcal{P}_i \wedge s' = s_T \\ 1 & \textrm{if } s = s_T \wedge s' = s_T \\ 0 & \textrm{if } s = s_T \wedge s' \neq s_T  \end{array} \right. ,
$$
the reward function
$$
R'(s, a) = \left\{ \begin{array}{cl} 
0 & \textrm{if } s = s_T \\ 

\sum\limits_{s' \in \mathcal{P}_i} \hspace{-0.5em} P(s'|s,a) R(s,a) & \textrm{if } s \neq s_T \wedge s' \neq s_T \\ 

\sum\limits_{y \in S \backslash \mathcal{P}_i} \hspace{-1.2em} \psi(s,a,y)  & \textrm{if } s \neq s_T \wedge s' = s_T \end{array} \right. \enspace ,
$$
where $\psi(s, a, y) = P(y|s,a)\left( R(s,a) + \gamma V_{M}(y) \right)$, and $\gamma$ is the discount factor from $M$.
\end{define}

\subsection{Proof of Theorem 1}

In this section, we prove Theorem 1. 

We will make use of the following notations. For $m \geq 1$, we will denote by $[m]$ the set $\{ 1, 2, \dots , m\}$. Let $\sigma = \langle \pi_\theta, \beta \rangle$ be a skill. Suppose the skill $\sigma$ is initialized from a state $s$. 
\begin{enumerate}
\item $P^{\pi_\theta}_{\beta}(s'|s,t)$ denotes the probability that the skill will terminate (i.e., return control to the agent) in state $s'$ exactly $t \geq 1$ timesteps after being initialized. 
\item $\widetilde{R}^{\pi_\theta}_{\beta,s}$ denotes the expected, discounted sum of rewards received during $\sigma$'s execution. We use the $\widetilde{\cdot}$ notation to emphasize that this quantity is discounted.
\end{enumerate}

The proof of Theorem 1 will make use of two lemmas. The first lemma (Lemma \ref{lem:update}) demonstrates a relationship between the value of a skill policy before and after replacing a single skill. Within the skill's partition class there is a $\gamma$-contraction (plus some error), but outside the skill's partition class the value may become worse by a bounded amount. The second lemma (Lemma \ref{lem:iteration}) uses Lemma \ref{lem:update} to prove that after a complete iteration (each skill has been update once), there is a $\gamma$-contraction (plus some error) over the entire state-space. We then prove Theorem 1 by applying the result of Lemma \ref{lem:iteration} recursively.

%
%
\begin{lemma} \label{lem:update}
Let 
\begin{enumerate}
\item $M$ be the target MDP, 
\item $\mathcal{P}$ a partition of the state-space, 
\item $\mu$ be the skill policy defined by $\mathcal{P}$ (i.e., $\mu(s) = \arg \max_{i \in [m]} \mathbb{I} \{ s \in \mathcal{P}_i \}$), 
\item $\Sigma$ be an ordered set of $m \geq 1$ skills, and 
\item $i \in [m]$ be the index of the $i^{\rm th}$ skill in $\Sigma$.
\end{enumerate}
Suppose we apply $\mathcal{A}$ to the Skill MDP $M_i'$ defined by $M$ and $V^{\langle \mu,\Sigma \rangle}_{M}$, obtain $\pi_\theta$, construct a new skill $\sigma_{i}' = \langle \pi_\theta, \beta_i \rangle$, and create a new skill set $\Sigma' = \left( \Sigma \backslash \{ \sigma_i \} \right) \cup \{ \sigma_{i}' \}$ by replacing the $i^{\rm th}$ skill with the new skill, then 
\begin{equation} \label{eqn:update_value_diff}
\displaystyle \forall_{s \in S} \enspace , \enspace V^{\langle \mu, \Sigma \rangle}_{M}(s) - V^{\langle \mu, \Sigma' \rangle}_{M}(s) \leq \frac{\eta}{1-\gamma}
\end{equation}
and
\begin{equation} \label{eqn:update_in_partition}
\displaystyle V^{*}_{M}(s) - V^{\langle \mu, \Sigma' \rangle}_{M}(s) \leq \left\{
\begin{array}{ll}
\gamma \left\| V^{*}_{M} - V^{\langle \mu, \Sigma \rangle}_{M} \right\|_{\infty} + \frac{\eta}{1-\gamma} & \textrm{if } s \in \mathcal{P}_i, \textrm{and} \\

\left\| V^{*}_{M} - V^{\langle \mu, \Sigma \rangle}_{M} \right\|_{\infty} + \frac{\eta}{1-\gamma} & \textrm{otherwise} ,
\end{array}
\right.
\end{equation}
where $\eta$ is the skill learning error.
\end{lemma}
\begin{proof}
~\\
%
%
\noindent {\bf Proving that (\ref{eqn:update_value_diff}) holds:} \\
First, we show that (\ref{eqn:update_value_diff}) holds. For each skill $\sigma_i \in \Sigma$, we will denote the skill's policy and termination rule by $\pi_i$ and $\beta_i$, respectively.  If $s \in \mathcal{P}_j$ where $j \neq i$, then 
$$
\begin{array}{rcl}
V^{\langle \mu, \Sigma \rangle}_{M}(s) - V^{\langle \mu, \Sigma' \rangle}_{M}(s) & = & \left( \widetilde{R}_{\beta_j ,s}^{\pi_j} + \sum\limits_{t=1}^\infty \gamma^t \sum\limits_{s'} P_{\beta_j}^{\pi_j}(s'|s,t)V^{\langle \mu, \Sigma \rangle}(s') \right) - \left( \widetilde{R}_{\beta_j ,s}^{\pi_j} + \sum\limits_{t=1}^\infty \gamma^t \sum\limits_{s'} P_{\beta_j}^{\pi_j}(s'|s,t)V^{\langle \mu, \Sigma' \rangle}(s') \right) \\

& \leq & \gamma \left\| V^{\langle \mu, \Sigma \rangle}_{M} - V^{\langle \mu, \Sigma' \rangle}_{M} \right\|_\infty \enspace .
\end{array}
$$
On the other hand, if $s \in \mathcal{P}_i$, then
%
%
\begin{align*}
V^{\langle \mu, \Sigma \rangle}_{M}(s) - V^{\langle \mu, \Sigma' \rangle}_{M}(s) &= V^{\langle \mu, \Sigma \rangle}_{M}(s) + \left(V^{*}_{M_{i}'}(s) - V^{*}_{M_{i}'}(s) \right) - V^{\langle \mu, \Sigma' \rangle}_{M}(s) && \text{ By inserting } 0 = \left(V^{*}_{M_{i}'}(s) - V^{*}_{M_{i}'}(s) \right) \enspace . \\
&= \left( V^{\langle \mu, \Sigma \rangle}_{M}(s) - V^{*}_{M_{i}'}(s) \right) + \left( V^{*}_{M_{i}'}(s) - V^{\langle \mu, \Sigma' \rangle}_{M}(s) \right) && \text{Regrouping terms.} \\
&\leq 0 + \left( V^{*}_{M_{i}'}(s) - V^{\langle \mu, \Sigma' \rangle}_{M}(s) \right) && \text{The definition of } M_{i}' \Rightarrow V^{*}_{M_{i}'}(s) \geq V^{\langle \mu, \Sigma \rangle}_{M}(s) \enspace . \\
&\leq \eta && \text{By Definition 4.}
\end{align*}
In either case, 
$$
V^{\langle \mu, \Sigma \rangle}_{M}(s) - V^{\langle \mu, \Sigma' \rangle}_{M}(s) \leq \gamma \left\| V^{\langle \mu, \Sigma \rangle}_{M} - V^{\langle \mu, \Sigma' \rangle}_{M} \right\|_\infty + \eta \enspace ,
$$
which leads to (\ref{eqn:update_value_diff}) by recursing on this inequality.

~\\
%
%
\noindent {\bf Proving that (\ref{eqn:update_in_partition}) holds:} \\
If $s \notin \mathcal{P}_i$, then by (\ref{eqn:update_value_diff}), we have
$$
\begin{array}{rcl}
V^{*}_{M}(s) - V^{\langle\mu,\Sigma'\rangle}_{M}(s) & \leq & V^{*}_{M}(s) - \left( V^{\langle\mu,\Sigma\rangle}_{M}(s) - \frac{\eta}{1-\gamma} \right) \\

& \leq & \left\| V^{*}_{M} - V^{\langle\mu,\Sigma\rangle}_{M} \right\|_\infty + \frac{\eta}{1-\gamma} \enspace .
\end{array}
$$
Now we consider the case where $s \in \mathcal{P}_i$. Let $\sigma_{i}' = \langle \pi_\theta, \beta_i \rangle$ be the newly introduced skill. We will denote by $\sigma_{i}'\langle\mu,\Sigma'\rangle$ the policy that first executes $\sigma_{i}'$ from a state $s \in \mathcal{P}_i$ and then follows the policy $\langle \mu, \Sigma \rangle$ thereafter.

$$
\begin{array}{lrcl}
\forall_{s \in \mathcal{P}_i} \enspace , & V^{*}_{M}(s) - V^{\langle\mu,\Sigma'\rangle}_{M}(s) & = & V^{*}_{M}(s) - V^{\sigma_{i}'\langle\mu,\Sigma'\rangle}_{M} \\

& & \leq & V^{*}_{M}(s) - \left( V^{\sigma_{i}^{*}\langle\mu,\Sigma'\rangle}_{M} - \eta \right) \\

& & = & \left( \widetilde{R}^{\pi^{*}}_{\beta_i} + \sum\limits_{t=1}^{\infty} \gamma^t \sum\limits_{s'} P^{\pi^{*}}_{\beta_i}(s'|s,t) V^{*}_{M}(s') \right) \\
& & & - \left( \widetilde{R}^{\pi^{*}}_{\beta_i} + \sum\limits_{t=1}^{\infty} \gamma^t \sum\limits_{s'} P^{\pi^{*}}_{\beta_i}(s'|s,t) V^{\langle\mu,\Sigma'\rangle}_{M}(s') \right) + \eta \\

& & = & \sum\limits_{t=1}^{\infty} \gamma^t \sum\limits_{s'} P^{\pi^{*}}_{\beta_i}(s'|s,t) \left( V^{*}_{M}(s') - V^{\langle\mu,\Sigma'\rangle}_{M}(s') \right) + \eta \\

& & \leq & \sum\limits_{t=1}^{\infty} \gamma^t \sum\limits_{s'} P^{\pi^{*}}_{\beta_i}(s'|s,t) \left( V^{*}_{M}(s') - \left( V^{\langle\mu,\Sigma\rangle}_{M}(s') - \frac{\eta}{1-\gamma} \right) \right) + \eta \\

& & \leq & \gamma \left\| V^{*}_{M} - V^{\langle\mu,\Sigma\rangle}_{M} \right\|_\infty + \gamma \left( \frac{\eta}{1-\gamma} \right) + \eta \\

& & \leq & \gamma \left\| V^{*}_{M} - V^{\langle \mu, \Sigma \rangle}_{M} \right\|_\infty + \frac{\eta}{1-\gamma} \enspace .
\end{array}
$$

\end{proof}

%
%
\begin{lemma} \label{lem:iteration}
Suppose we execute \Alg\ for a single iteration. Let $\Sigma$ be the set of skills at the beginning of the iteration and $\Sigma'$ be the set of skills after each skill has been updated and the iteration has completed, then 
\begin{equation} \label{eqn:lsb:iteration}
\left\| V^{*}_{M} - V^{\langle \mu, \Sigma' \rangle}_{M} \right\|_\infty \leq \gamma \left\| V^{*}_{M} - V^{\langle \mu, \Sigma \rangle}_{M} \right\|_\infty + \frac{m\eta}{1-\gamma} \enspace .
\end{equation}
\end{lemma}
\begin{proof}
Without loss of generality, we assume that the skills are updated in order of increasing index. We denote the skill set at the beginning of the iteration by $\Sigma$ and the skill set at the end of the iteration by $\Sigma'$ after all of the skills have been updated once. It will be convenient to refer to the intermediate skill sets that are created during an iteration. Therefore, we denote by $\Sigma_1', \Sigma_2', \dots, \Sigma_m' = \Sigma'$, the set of skills after the first skill was replaced, the second skill was replaced, \dots , and after the $m^{\rm th}$ skill was replaced, respectively.

We will proceed by induction on the skill updates. As the base case, notice that by Lemma \ref{lem:update}, we have that
$$
\displaystyle V^{*}_{M}(s) - V^{\langle \mu, \Sigma_1' \rangle}_{M}(s) \leq \left\{
\begin{array}{ll}
\gamma \left\| V^{*}_{M} - V^{\langle \mu, \Sigma \rangle}_{M} \right\|_{\infty} + \frac{\eta}{1-\gamma} & \textrm{if } s \in \mathcal{P}_1, \textrm{and} \\

\left\| V^{*}_{M} - V^{\langle \mu, \Sigma \rangle}_{M} \right\|_{\infty} + \frac{\eta}{1-\gamma} & \textrm{otherwise} .
\end{array}
\right.
$$

Let $1 \leq i < m$. Now suppose for $\Sigma_i'$, we have that
$$
\displaystyle V^{*}_{M}(s) - V^{\langle \mu, \Sigma_i' \rangle}_{M}(s) \leq \left\{
\begin{array}{ll}
\gamma \left\| V^{*}_{M} - V^{\langle \mu, \Sigma \rangle}_{M} \right\|_{\infty} + \frac{i \eta}{1-\gamma} & \textrm{if } s \in \bigcup\limits_{j\in [i]} \mathcal{P}_j, \textrm{and} \\

\left\| V^{*}_{M} - V^{\langle \mu, \Sigma \rangle}_{M} \right\|_{\infty} + \frac{i \eta}{1-\gamma} & \textrm{otherwise} .
\end{array}
\right.
$$

Now for $\Sigma_{i+1}'$, we have several cases:
\begin{enumerate}
\item $s \in \mathcal{P}_{i+1}$:

By applying Lemma \ref{lem:iteration}, we see that

$$
\begin{array}{rcl}
V^{*}_{M}(s) - V^{\langle \mu, \Sigma_{i+1}' \rangle}_{M}(s) & \leq & \gamma \left\| V^{*}_{M} - V^{\langle \mu, \Sigma_{i}' \rangle}_{M} \right\|_\infty + \frac{\eta}{1-\gamma} \\

& \leq & \gamma \left( \left\| V^{*}_{M} - V^{\langle \mu, \Sigma \rangle}_{M} \right\|_\infty + \frac{i \eta}{1-\gamma} \right) + \frac{\eta}{1-\gamma} \\

& \leq & \gamma \left\| V^{*}_{M} - V^{\langle \mu, \Sigma \rangle}_{M} \right\|_\infty + \frac{(i+1) \eta}{1-\gamma} \enspace .
\end{array}
$$

\item $s \in \bigcup\limits_{j \in [i]} \mathcal{P}_{j}$:

By applying Lemma \ref{lem:iteration}, we see that

$$
\begin{array}{rcl}
V^{*}_{M}(s) - V^{\langle \mu, \Sigma_{i+1}' \rangle}_{M}(s) & \leq & V^{*}_{M}(s) - V^{\langle \mu, \Sigma_{i}' \rangle}_{M}(s) + V^{\langle \mu, \Sigma_{i}' \rangle}_{M}(s) - V^{\langle \mu, \Sigma_{i+1}' \rangle}_{M}(s) \\

& \leq & V^{*}_{M}(s) - V^{\langle \mu, \Sigma_{i}' \rangle}_{M}(s) + \frac{\eta}{1-\gamma} \\

& \leq & \left( \gamma \left\| V^{*}_{M} - V^{\langle \mu, \Sigma \rangle}_{M} \right\|_{\infty} + \frac{i \eta}{1-\gamma} \right) + \frac{\eta}{1-\gamma} \\

& = & \gamma \left\| V^{*}_{M} - V^{\langle \mu, \Sigma \rangle}_{M} \right\|_{\infty} + \frac{(i+1) \eta}{1-\gamma} \enspace .
\end{array}
$$

\item $s \notin \bigcup\limits_{j \in [i+1]} \mathcal{P}_{j}$:

Again, by Lemma \ref{lem:iteration}, we see that

$$
\begin{array}{rcl}
V^{*}_{M}(s) - V^{\langle \mu, \Sigma_{i+1}' \rangle}_{M}(s) & \leq & \left\| V^{*}_{M} - V^{\langle \mu, \Sigma_{i}' \rangle}_{M} \right\|_{\infty} + \frac{\eta}{1-\gamma} \\

& \leq & \left( \left\| V^{*}_{M} - V^{\langle \mu, \Sigma \rangle}_{M} \right\|_{\infty} + \frac{i \eta}{1-\gamma} \right) + \frac{\eta}{1-\gamma} \\

& \leq & \left\| V^{*}_{M} - V^{\langle \mu, \Sigma \rangle}_{M} \right\|_{\infty} + \frac{(i+1) \eta}{1-\gamma} \enspace .
\end{array}
$$
\end{enumerate}

Thus by the principle of mathematical induction the statement is true for $i = 1, 2, \dots , m-1$. After performing $m$ updates, $\bigcup\limits_{j \in [m]} \mathcal{P}_j \equiv S$. Thus, we obtain the $\gamma$-contraction over the entire state-space.

\end{proof}

%
%
\subsubsection{Proof of Theorem 1}

\begin{proof} {\bf (of Theorem 1)}

The loss of all policies is bounded by $\frac{1}{1-\gamma}$. Therefore, by applying Lemma \ref{lem:iteration} and recursing on (\ref{eqn:lsb:iteration}) for $K \geq \log_{\gamma} \left( \varepsilon ( 1-\gamma ) \right)$ iterations, we obtain
$$
\begin{array}{rcl}
\left\| V^{*}_{M} - V^{\varphi}_{M} \right\|_\infty & \leq & \gamma^{K} \left( \frac{1}{1-\gamma} \right) + \frac{m\eta}{(1-\gamma)^{2}} \\

& \leq & \gamma^{\log_{\gamma} ( \varepsilon ( 1-\gamma ) )} \left( \frac{1}{1-\gamma} \right) + \frac{m\eta}{(1-\gamma)^{2}} \\

& = & ( \varepsilon ( 1-\gamma ) ) \left( \frac{1}{1-\gamma} \right) + \frac{m\eta}{(1-\gamma)^{2}} \\

& = & \frac{m\eta}{(1-\gamma)^{2}} + \varepsilon \enspace .
\end{array}
$$

\end{proof}

%
%
\subsection{Experiments}
We performed experiments on three well-known RL benchmarks: Mountain Car (MC), Puddle World (PW) \cite{Sutton1996} and the Pinball domain \cite{Konidaris2009}. The MC domain is discussed here. The PW and Pinball domains are found in the main paper. The purpose of our experiments is to show that \Alg\ can solve a complicated task with a simple policy representation by combining bootstrapped skills. These domains are simple enough that we can still solve them using richer representations. This allows us to compare \Alg\ to a policy that is very close to optimal. Our experiments demonstrate potential to scale up to higher dimensional domains by combining skills over simple representations.

Recall that \Alg\ is a meta-algorithm. We must provide an algorithm for Policy Evaluation (PE) and skill learning. In our experiments, for the MC domain, we used SMDP-LSTD \cite{Sorg2010} for PE and a modified version of Regular-Gradient Actor-Critic \cite{Bhatnagar2009} for skill learning.

In our experiments, for the MC domain, each skill is simply represented as a probability distribution over actions (independent of the state). We compare the performance to a policy using the same representation that has been derived using the monolithic approach. Each experiment is run for $10$ independent trials. A $2 \times 2$ grid partitioning is used for the skill partition in this domain, unless otherwise stated. Binary-grid features are used to estimate the value function. 

These are example representations. In principal, any value function and policy representation that is representative of the domain can be utilized. 

\subsubsection{Mountain Car}
\label{exp:mc}
The Mountain Car domain consists of an under-powered car situated in a valley. The car has to leverage potential energy to propel itself up to the goal, which is the top of the rightmost hill. The state-space is the car's position and velocity $\langle p, v \rangle$.

Figure \ref{fig:mc}$a$ compares the monolithic approach with \Alg\ (for a $2 \times 2$ grid partition). The monolithic approach achieves low average reward. However, with the same restricted policy representation, \Alg\ combines a set of skills, resulting in a richer solution space and a higher average reward as seen in Figure \ref{fig:mc}$a$. This is comparable to the approximately optimal average reward. Convergence is achieved after a single iteration since \Alg\ is initiated from the partition containing the goal location causing value to be instantaneously propagated to subsequent skills.


Figure \ref{fig:mc}$b$ compares the performance of different partitions where a $1 \times 1$ grid represents the monolithic approach. As seen in the figure, the cost is lower for all partitions greater than $1 \times 1$ which is consistent with the results in the main paper. Figure \ref{fig:vfmc} indicates the resulting value functions for various grid sizes. The value function from the monolithic approach is not capable of solving the task whereas the $4 \times 4$ grid partition's value function is near-optimal.



\begin{figure*}
\centering
\begin{tabular}{ccc}
\includegraphics[width=0.5\textwidth]{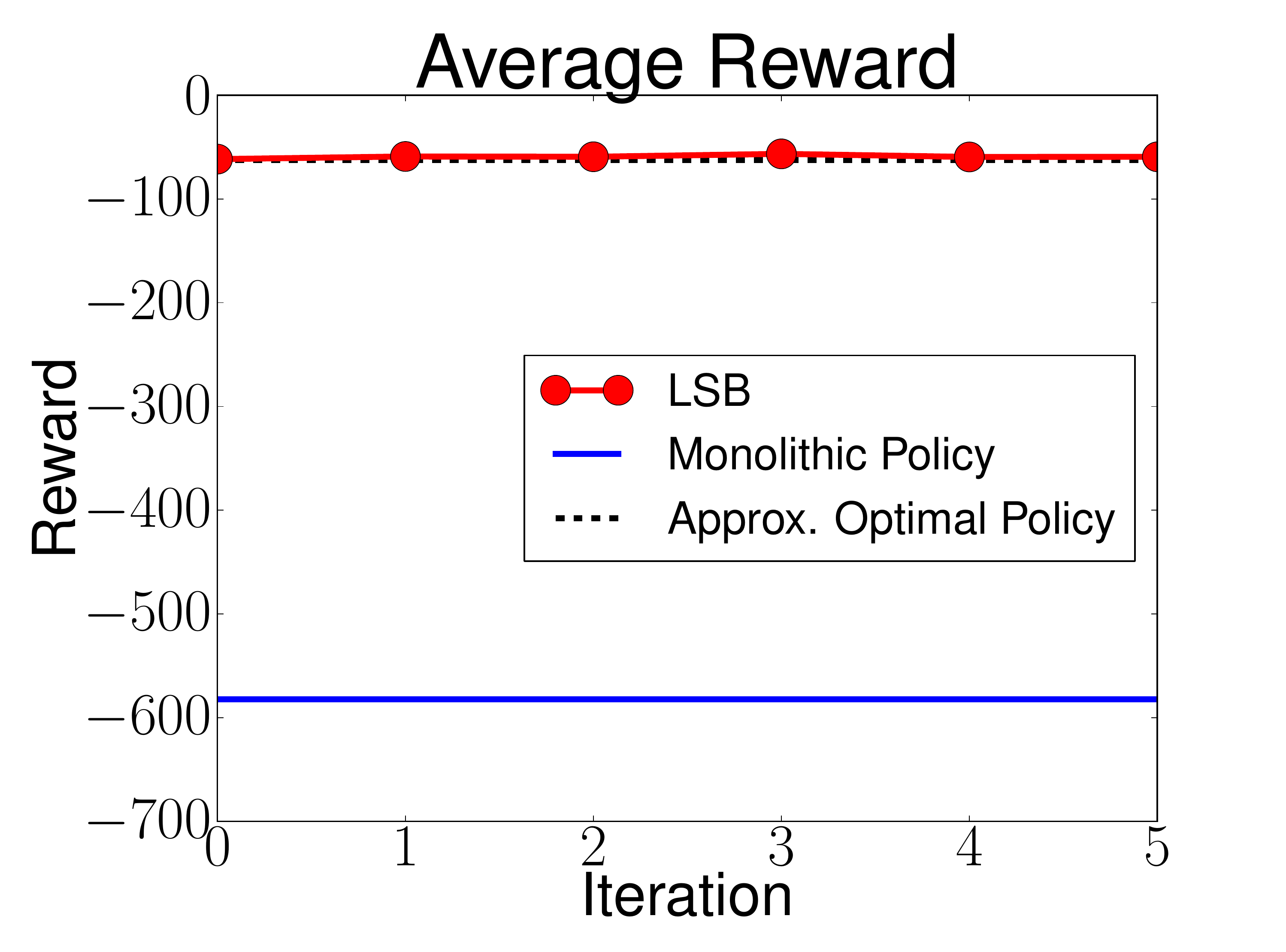} & \hspace{0em} &
\includegraphics[width=0.5\textwidth]{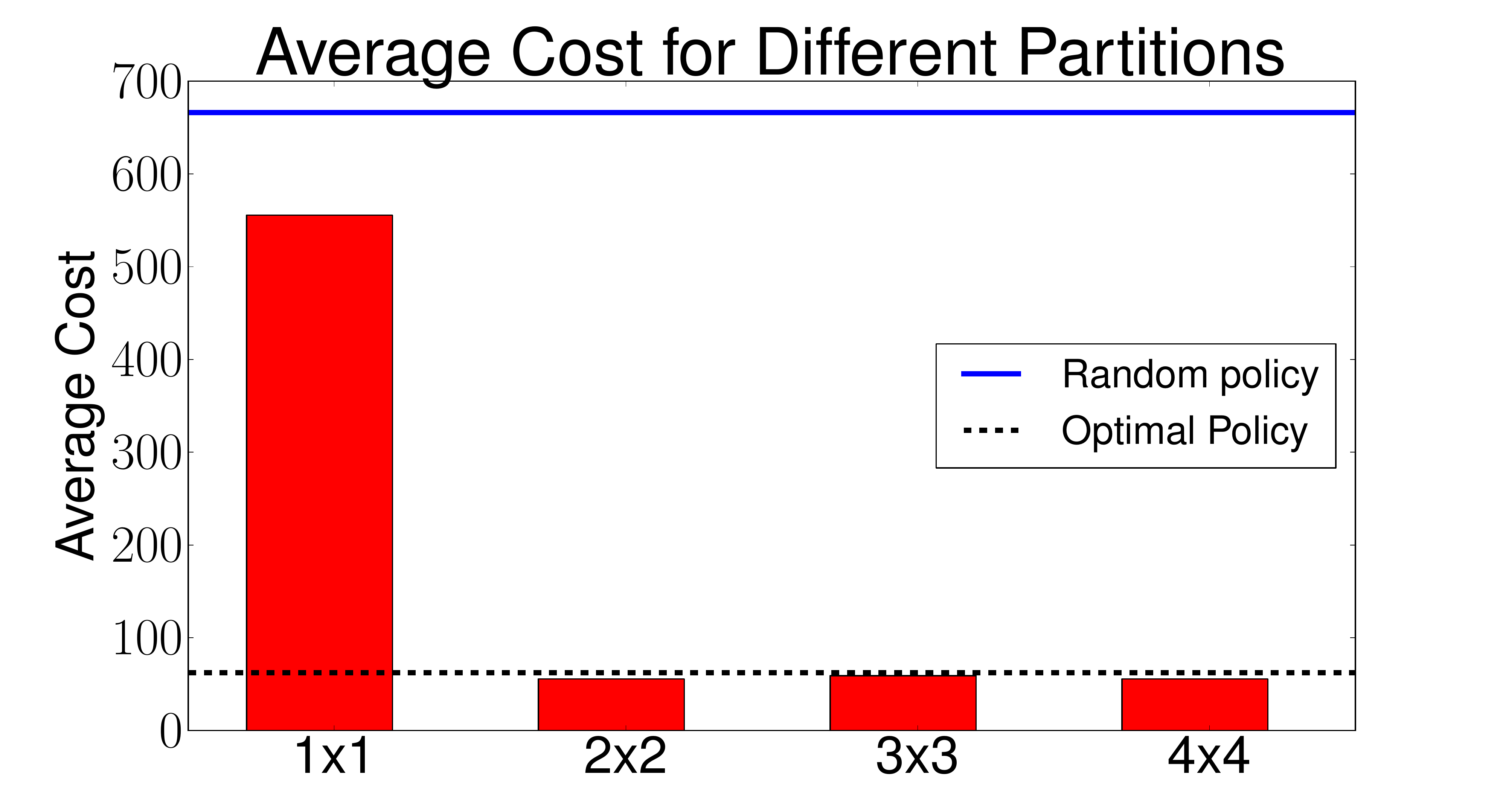} \\
($a$) & & ($b$)
\end{tabular}
\caption{The Mountain Car domain: ($a$)  The average reward for the \Alg\ algorithm generated by the \Alg\ skill policy. This is compared to the monolithic approach that attempts to solve the global task as well as an approximately optimal policy derived using Q-learning. ($b$) The average cost (negative reward) for different partitions (i.e., grid sizes).}
\label{fig:mc}
\end{figure*}

\begin{figure}
\centering
\includegraphics[width=0.7\textwidth]{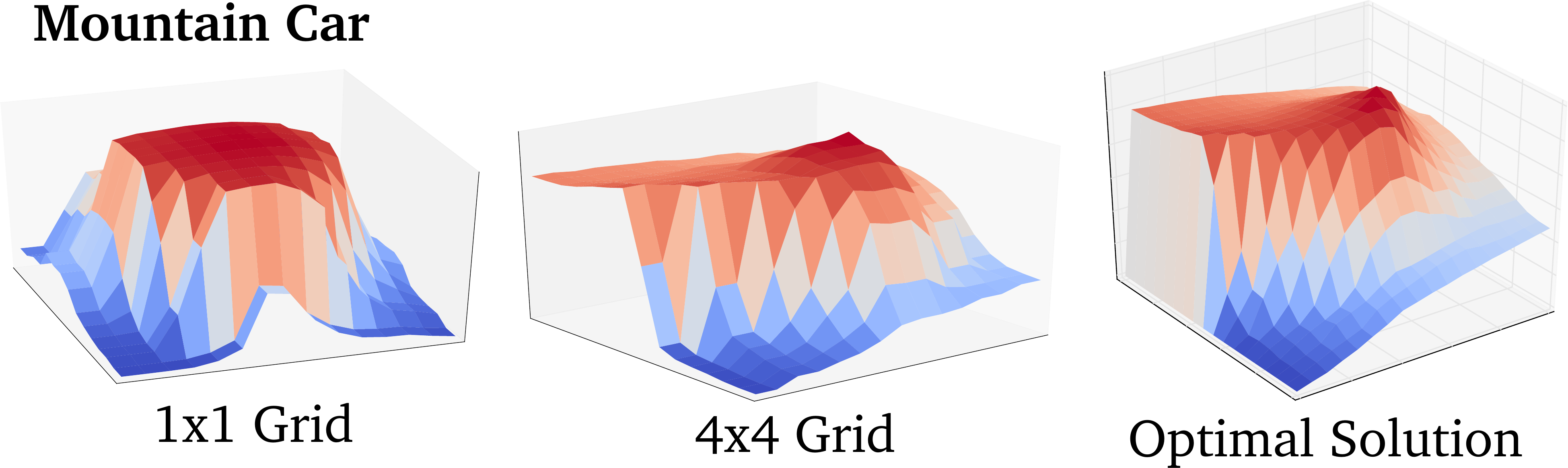}
\caption{\textit{The Mountain Car Domain}: Comparison of value functions for the monolithic approach ($1 \times 1$ grid), the best partition using \Alg\ ($4 \times 4$ grid), and an approximately optimal value function (derived using Q-learning, applied for a huge number of iterations, with a fine discretization of each task's state-space).}
\label{fig:vfmc}
\end{figure}

\subsubsection{Puddle World}
Figure \ref{fig:vf} compares the value functions for different grid sizes in Puddle World. The monolithic approach ($1 \times 1$ partition) provides a highly sub-optimal solution since, according to its value function, the agent must travel directly through the puddles to reach the goal location. The $3 \times 3$ grid provides a near-optimal solution.

\begin{figure}
\centering
\includegraphics[width=0.8\textwidth]{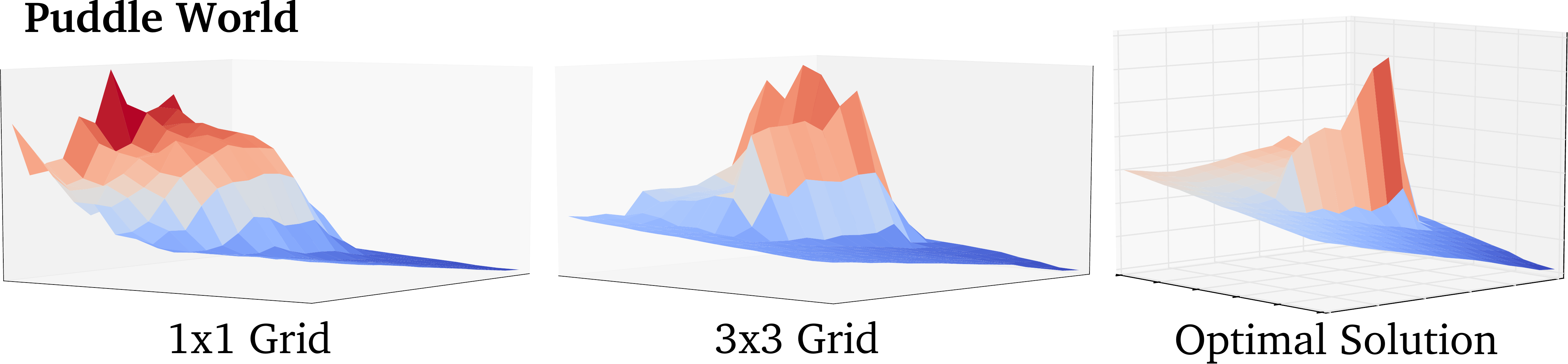}
\caption{\textit{The Puddle World Domain}: Comparison of value functions for the monolithic approach ($1 \times 1$ grid), the best partition using \Alg\ ($3 \times 3$ grid), and an approximately optimal value function (derived using Q-learning, applied for a huge number of iterations, with a fine discretization of each task's state-space).}
\label{fig:vf}
\end{figure}

%
%
\subsection{Modified Regular-Gradient Actor–Critic}

We used a very simple policy gradient algorithm (Algorithm \ref{alg:vpg}) for skill learning. The algorithm is based on Regular-Gradient Actor–Critic \cite{Bhatnagar2009}. The algorithm differs from Regular-Gradient Actor–Critic because it uses different representations for approximating the value function and the policy. For a state action pair $(s, a) \in S \times A$, a functions $\phi(s,a) \in \mathbb{R}^{d}$ and $\zeta(s,a) \in \mathbb{R}^{d'}$ mapped $(s,a)$ to a vector with dimension $d$ and $d'$, respectively. We use the representation given by $\phi$ to approximate the value function and the representation given by $\zeta$ to represent the policy. The parametrized policy was defined by
\begin{equation} \label{eqn:policy}
\pi_{\theta}(a|s) = \frac{ \exp \left( \theta^{T} \zeta(s,a) \right) }{ \sum\limits_{a' \in A} \exp \left( \theta^{T} \zeta(s,a') \right) } \enspace ,
\end{equation}
where $\theta \in \mathbb{R}^{d'}$ are the learned policy parameters. We used representations such that $d' \ll d$, meaning that the policy parametrization was much simpler than the representation used to approximate the value function. This allowed us to get an accurate representation of the value function, but restrict the policy space to very simple policies.

\begin{algorithm}
\caption{Modified Regular-Gradient Actor–Critic}
\label{alg:vpg}
\begin{algorithmic}[1]
\REQUIRE ~\\
\begin{enumerate}
\item $\phi$ : mapping from states to a vector representation used to approximate the value, 
\item $\omega$ : value function approximation parameters,
\item $\zeta$ : mapping from states to a vector representation used to approximate the policy,
\item $\theta$ : policy parameters,
\item $\alpha$ : the value learning rate (fast learning rate),
\item $\beta$ : the policy learning rate (slow learning rate, i.e., $\beta < \alpha$), and
\item $(s, a, s', r)$ : a state-action-next-state-reward tuple.
\end{enumerate}
\STATE $\widehat{V}_{\rm NEW} \leftarrow \left( r + \gamma \sum\limits_{a' \in A} \pi_{\theta}(a'|s') \omega^{T} \phi(s,a') \right)$ \COMMENT{Estimate $V^{\pi_\theta}$ given the new sample.}
\STATE $\widehat{V}_{\rm OLD} \leftarrow \omega^{T} \phi(s,a)$ \COMMENT{Use the current value function approximation to estimate the value.}
\STATE $\delta \leftarrow \left( \widehat{V}_{\rm NEW} - \widehat{V}_{\rm OLD} \right)$ \COMMENT{Compute the temporal difference error.}
\STATE $\omega' \leftarrow \omega + \alpha \delta$ \COMMENT{ Update the value function weights using the fast learning rate $\alpha$. }
\STATE $\psi_{s,a} \leftarrow \zeta(s,a) - \sum\limits_{a' \in A} \pi_\theta(a',s)\zeta(s,a')$ \COMMENT{Compute the ``compatible features'' \cite{Bhatnagar2009}. }
\STATE $\theta' \leftarrow \theta + \beta \delta \psi_{s,a}$ \COMMENT{ Update the policy parameters using the slow learning rate $\beta$. }
\STATE \textbf{Return} $\langle \omega', \theta' \rangle$ \COMMENT{ Updated value function and policy parameters. }
\end{algorithmic}
\end{algorithm}

Although using different representations for approximating the value function and the policy strictly violates the policy gradient theorem \cite{Sutton2000}, it still tends to work well in practice.

In our experiments, we used a fast learning rate of $\alpha = 0.1$ and a slow learning rate of $\beta = 0.2 \alpha$. Value function and policy parameters were initialized to zero vectors.

%
%
\subsection{Pinball Demonstration Videos}

There are two videos attached showing a demonstration of a policy learned by LSB for the Pinball domain \cite{Konidaris2009}. Both of these domains are analyzed and discussed in the main paper. The first video shows a policy learned for \textit{Maze-world} and the second video shows a policy learned for \textit{Pinball-world}, one of the standard pinball benchmark domains. The objective of the agent (blue ball) is to circumnavigate the obstacles and reach the goal region (red ball). A colored square is superimposed onto the active skill indicating the current skill or partition class being executed.
%

%
%
%
\bibliography{tmann}
\bibliographystyle{plainnat}

\end{document}